\documentclass{article}

\usepackage[margin=1in]{geometry}
\usepackage{amsmath,amssymb,amsthm} % math notation and theorem environments
\usepackage{xfrac}          % compact fractions
\usepackage[utf8]{inputenc} % allow utf-8 input
\usepackage[T1]{fontenc}    % use 8-bit T1 fonts
\usepackage[colorlinks]{hyperref}       % hyperlinks
\usepackage{url}            % simple URL typesetting
\usepackage{booktabs}       % professional-quality tables
\usepackage{nicefrac}       % compact symbols for 1/2, etc.
\usepackage{microtype}      % microtypography
\usepackage{xcolor}         % comment text in different color
\usepackage{enumitem}       % customized enumerate environments
\usepackage{algorithm}
\usepackage{algorithmic}
\usepackage{subcaption}
\usepackage{multicol}

\usepackage{amsfonts}
\usepackage{tikz-cd}
\usepackage{color,soul}
\usepackage{todonotes}
\usepackage{array}
\usepackage{zref-savepos}
\usepackage{mathtools}
\usepackage{amsthm}
\usepackage{multirow, bigstrut}
\usepackage[font={small,it}]{caption}
\usepackage{tikz}
\usetikzlibrary{arrows}

\newtheorem{theorem}{Theorem}[section]
\newtheorem*{theorem*}{Theorem}
\newtheorem*{question*}{Question}
\newtheorem{lemma}[theorem]{Lemma}
\newtheorem*{lemma*}{Lemma}

\newtheorem{corollary*}{Corollary}

\newtheorem*{problem*}{Problem}

\theoremstyle{definition}
\newtheorem{definition}[theorem]{Definition}

\newtheorem*{remark*}{Remark}

\newcommand{\R}{\mathbb{R}}

\newcommand{\X}{\mathbf{X}}
\newcommand{\W}{\mathbf{W}}
\newcommand{\Var}{\text{Var}}

\DeclareMathOperator{\E}{\mathbb{E}}

\DeclarePairedDelimiter\parentheses{\lparen}{\rparen}
\newcommand{\Tr}[1]{\operatorname{Tr} \parentheses*{#1}}
\newcommand{\diag}[1]{\operatorname{diag} \parentheses*{#1}}
\newcommand{\ip}[1]{\left \langle #1 \right \rangle}

\title{Asymmetric Random Projections}
\date{}
\author{%
  Nick Ryder\thanks{This work was done while Nick Ryder was an intern at Amazon, Palo Alto.} \\
  Department of Mathematics\\
  University of California Berkeley, Berkeley, CA, USA\\
  \texttt{nick.ryder@berkeley.edu} \\
  Zohar Karnin \\
  Amazon, New York \\
  \texttt{zkarnin@amazon.com} \\
  Edo Liberty \\
  Amazon, Palo Alto \\
  \texttt{eliberty@amazon.com} \\
}

\begin{document}

\maketitle

\begin{abstract}
Random projections (RP) are a popular tool for reducing dimensionality while preserving local geometry. In many applications the data set to be projected is given to us in advance, yet the current RP techniques do not make use of information about the data. In this paper, we provide a computationally light way to extract statistics from the data that allows designing a data dependent RP with superior performance compared to data-oblivious RP. We tackle scenarios such as matrix multiplication and linear regression/classification in which we wish to estimate inner products between pairs of vectors from two possibly different sources. Our technique takes advantage of the difference between the sources and is provably superior to oblivious RPs. Additionally, we provide extensive experiments comparing RPs with our approach showing significant performance lifts in fast matrix multiplication, regression and classification problems.
\end{abstract}

\section{Introduction}
\label{sec:intro}

The use of random projections (RPs) as a method for data compression has been well studied for the past few decades. RPs provide a theoretical backing for oblivious compression techniques and are employed in multiple scenarios such as sketching \cite{alon1999space}, linear regression and other classification tasks \cite{fradkin2003experiments}, $k$-nearest neighbors \cite{andoni2006near}, fast approximate linear algebra \cite{clarkson2009numerical}, and more. One important property of RPs is that they are oblivious, meaning they can be determined without observing data. While this is useful for many applications, this restriction is unnecessary in several realistic scenarios.

% The use of random projections as a method for data compression has been well studied for the past few decades. Random projections provide a theoretical backing for oblivious compression techniques. These techniques are employed in scenarios where we want to quickly compress data to get significantly smaller model sizes at the cost of well understood error. From this perspective, there are multiple aspects to optimize over: sparsity, randomness, and mean error. Much has been done to study sparse random projections, with modern algorithms reaching tightness results. Likewise, people have developed methods which use linearly many bits to obtain pseudo-random projections which still give high probability guarantees on the compressed data. In all of these frameworks, we still approach random projections from a data-oblivious point of view. \zk{change first paragraph. Mention that RP is used in ..., and that many papers improved the techniques over the years (a few citations), then mention that most of them are data oblivious.}

In this paper we tackle the problem of obtaining a data-dependent random projection. We measure the quality of a random projection by its ability to preserve the inner product of a vector pair $x,w$. Specifically, we view a random projection as a tool to obtain a random estimate of $\ip{x,w}$ by taking the inner product of their projected version. Under the requirement of unbiasedness we aim to minimize variance, as common when dealing with estimators. An oblivious approach must handle the worst case scenario for $x,w$. However, with access to data we find ourselves in a setting where $x$ and $w$ are random vectors coming from some distribution. This occurs in matrix multiplication where $x,w$ are columns of two matrices. It also occurs in linear regression or classification where $x$ is a random data point and $w$ is a regressor chosen at random from some prior distribution.

With $x,w$ being random, the variance associated with the random projection is now over the randomness of both the projection itself, and of $x,w$. In what follows we analyze the \emph{optimal} linear pre-processing for $x$ and $w$ that improves the performance of an oblivious random projection applied to the processed versions. By choosing this black-box approach we  have the flexibility to use any of the well-known random projection methods, including sparse random projections~\cite{nelson2013osnap} or FJLT~\cite{ailon2009fast}, as long as they have associated to them JL-lemma guarantees~\cite{johnson1984extensions}. In addition to the optimal pre-processing transformation we provide a linear-time (in the input size) counterpart. We analyze its performance and show that it is never inferior and often strictly better than performing no pre-processing.

We apply our technique to applications of oblivious random projections. We show how our methods can be used for approximate fast matrix multiplication in a straightforward way. Another application of our methods is for linear regression and classification on high dimensional data. Our data-dependent random projection gives rise to a novel technique that includes the standard way of using oblivious random projections for classification/regression, yet can tune itself to the input data and improve the quality of the model, with negligible computation overhead. We empirically test our algorithm with an extensive set of experiments. For approximate matrix multiplication tasks we achieve a  reduction of 50\% to 60\% to the MSE with near-zero computational overhead compared with oblivious random projection, in multiple real datasets. For linear regression and classification, when compared to oblivious random projections we achieve a lift of 4\% in accuracy for binary classification on the {\bf RCV1} dataset, and a decrease of 61.2\% to MSE for the regression task in the {\bf Slice Localization} dataset.

%%%%%%%%%%%%%%%%%%%%%%%%
%%%%%%%%%%%%%%%%%%%%%%%%
%%%%%%%%%%%%%%%%%%%%%%%%
%
% 
%
%%%%%%%%%%%%%%%%%%%%%%%%
%%%%%%%%%%%%%%%%%%%%%%%% 
%%%%%%%%%%%%%%%%%%%%%%%%
  \vspace{-0.05in}

\section{Comparison with Previous Results}
  \vspace{-0.05in}

Random projections (RPs) have been originally proposed by Johnson and Lindenstrauss \cite{johnson1984extensions}. The projection they offer is linear, which in turn is useful for several applications such as sketching \cite{alon1999space}, linear regression and other classification tasks \cite{fradkin2003experiments}, $k$-nearest neighbors \cite{andoni2006near}, fast approximate linear algebra \cite{clarkson2009numerical}, and more. These projections were simplified and improved with time  \cite{dasgupta2003elementary,achlioptas2003database,ailon2009fast,nelson2013osnap}.

% Random projections have been originally proposed by Johnson and Lindenstrauss \cite{johnson1984extensions}. They prove that the linear projection of a point onto a random subspace is a mapping that approximately preserves the norm of any vector. As a corollary, given a collection of points, their projection onto a linear subspace provides a low space representation that approximately preserves the geometry of the original set of points. The fact that the mapping is linear makes it useful for several applications such as sketching \cite{alon1999space}, linear regression and other classification tasks \cite{fradkin2003experiments}, $k$-nearest neighbors \cite{andoni2006near}, fast approximate linear algebra \cite{clarkson2009numerical}, and more.

% Due to the many application they admit, research produced various different methods for obtaining random projections. In \cite{dasgupta2003elementary} the authors show that random Gaussian matrices provide the same guarantees, up to lower order terms as random subspaces. This proof was extended in \cite{achlioptas2003database} showing that a matrix with random signs is sufficient. A series of papers starting with \cite{ailon2009fast} provided a random projection matrix that can be applied in near-linear time regardless of the target dimension, by applying the Hadamard transform on the data. A different line of work \cite{nelson2013osnap} pursued sparse random projection matrices; this are useful for computational speed when the input vectors are sparse themselves.

Applications of random projections typically use the fact that for a collection of input vectors, the collection of low dimensional vectors obtained via the random projection have the same pairwise distances, or inner products. 
%For $k$-nearest neighbors and approximate matrix multiplication the application is immediate \zk{citations for knn and fmm}. 
For linear regression one can use the guarantee that $\ip{x,w^*}$ is preserved for every data point $x$ and the optimal (unknown) regressor $w^*$. This was used for developing a fast algorithm for linear regression \cite{fradkin2003experiments, maillard2012linear}. This property was used for kernel based linear classification in \cite{rahimi2008random} that choose random features from the kernel space, thereby removing the quadratic dependence on the example number in the run-time of SVM. Another application example is fast matrix multiplication. Given two matrices $X, W$, we can view their matrix product $X^\top W$ as encoding all the pairwise inner products of the columns of $X$ with the columns of $W$. From this perspective coming up with approximate fast matrix multiplication algorithms is equivalent to quickly estimating the inner products between these columns. One approach is to treat the columns as two separate data sets and compress them with the objective of minimizing the distortion of their inner products \cite{fmm-project, fmm-sketching}. 

The RP methods mentioned above are oblivious to the data. Other techniques provide data-dependent solutions with improved guarantees. A classic example is PCA, or CCA for when inner products are applied to different sets of vectors. These methods provide a deterministic guarantee, but come with a heavy computational cost; even the approximate version of PCA, CCA (see e.g.\ \cite{karnin2015online}) are never (quasi-)linear in the input dimension (as opposed to FJLT~\cite{ailon2009fast}). Additionally, the bias of the error may be problematic when the objective is not minimizing the \emph{mean} error. There are several other deterministic methods for dimensionality reductions, with different objectives than preserving inner products listed in the survey \cite{cunningham2015linear}.

Another data dependent approach consists of storing the norms of the original vectors in addition to their random projections. In \cite{li2006improving} the authors compute the MLE of the inner product or distance based on the RP approximation and the norm. In \cite{kang2017control} the norms are used to reduce the variance of the RP estimate. These methods are complementary to ours given that we modify the random projection rather than store additional information. One advantage of our technique is that it can be applied based on a sample of the data, rather than having a hard requirement of observing the entire data in advance. This is key in the application of linear regression / classification (Section~\ref{sec:linear}).  In \cite{cohen2015dimensionality} the authors provide non-oblivious random projections obtained by distorting the space according to the covariance matrix of the data. Specifically, they propose to multiply the data matrix with a random projection, then orthogonalize the result. The estimates of inner products are no longer unbiased, but the authors show that this non-oblivious projection provides better guarantees for $k$-means clustering, and $k$-rank approximation. The authors of \cite{sen2013informed} use a mixture of PCA and random projections to obtain a data dependent projection that can potentially have superior guarantees to oblivious random projections. 

  \vspace{-0.05in}

\section{Data Dependent Random Projections} \label{sec:ddrp}
  \vspace{-0.05in}

In what follows we consider the following setup. There are two distributions $\X, \W$ over vectors of dimension $d$ that are known to us, either completely or via oracle access to i.i.d. samples. We wish to estimate inner products of the form $\ip{x,w}$ where $x \sim \X, w \sim \W$. We do so via a linear dimension reduction operator that transforms $x$ and $w$ to dimension $k \ll d$ vectors $\tilde{x}, \tilde{w}$ in a way that $\ip{\tilde{x}, \tilde{w}} \approx \ip{x,w}$.

\vspace{-0.05in}
\subsection{Oblivious Random Projections}
\vspace{-0.05in}
Consider the case of an oblivious random projection. Here $R$ is a random $k \times d$ matrix and we have $\tilde{x} = Rx, \tilde{w} = Rw$. We consider the random variable $\ip{\tilde{x}, \tilde{w}}$ as an estimate to $\ip{x,w}$. In what follows we provide an asymmetric pre-processing step for $x$ and $w$ applied before the random projection. We analyze its guarantee for any random projection giving an unbiased estimate with a specific variance bound. This is formally defined here

\begin{definition}
A \emph{valid random projection} $R$ mapping $d$ dimensions into $k$ is one that for some constant $C$ independent of the data or its 
dimension, is such that 
$$ \E[\ip{Rx,Rw}] = \ip{x,w}, \ \ \ \ \Var[\ip{Rx,Rw}] \leq \frac{C \ip{x,w}^2 + \|x\|^2\|w\|^2}{k} $$
\end{definition}

As a sanity check we mention that the standard random projections are indeed \emph{valid random projections}. For completeness we prove this for $R$ obtained from i.i.d entries (the proof is deferred to Appendix~\ref{sec:proof_ddrp}). A similar statement can be made for other random projections yet is outside the scope of this paper.
\begin{lemma} \label{lem:iid_rp}
Let $R$ be a $k \times d$ matrix of i.i.d entries whose first 4 moments are $0,1/k,0,s_4$, with $s_4 \leq 3/k^2$. We have that $\ip{\tilde{w},\tilde{x}}$ is an unbiased estimator of $\ip{x,w}$ whose variance is at most $\frac{\ip{x,w}^2 + \|x\|^2\|w\|^2}{k}$.
\end{lemma}

\vspace{-0.1in}
\subsection{Our Solution}
\vspace{-0.05in}
Our technique follows from a simple observation. Consider an invertible matrix $A$. We choose a different projection for the vector $x$ and $w$. Specifically, we set 
$$ \tilde{x} = RAx, \ \ \ \tilde{w} = RA^{-\top}w $$
where $A^{-\top} = (A^{-1})^\top$ is the inverse transpose of $A$. For these estimate it is easy to observe that $\ip{\tilde{x},\tilde{w}}$ remains an unbiased
estimate of $x,w$ since $\ip{Ax, A^{-\top}w}=\ip{x,w}$. However, when we use the variance bound for valid random projections we get 
$$\Var(x,w) \leq C \ip{x,w}^2 + \|Ax\|^2\|A^{-\top}w\|^2$$
meaning we replaced the term $\|x\|^2\|w\|^2$ with $\|Ax\|^2 \|A^{-\top}w\|^2$. Notice that unless $x$ and $w$ have very close directions the $\|x\|^2\|w\|^2$
term is the dominant one in the variance bound.  Now, since our vectors $x,w$ are drawn from known distributions we can consider a matrix $A$ that minimizes that
quantity when averaged over the possible draws. Specifically, we aim to minimize the function
$$\Phi(A) = \E_{x,w}\left[\|Ax\|^2\|A^{-\top}w\|^2\right]$$
\vspace{-0.02in}
It turns out that $\Phi$ can be efficiently minimized by applying the technique of CCA (Canonical Correlation Analysis) on the covariance matrices.

\begin{theorem} \label{thm:full}
Let $\X,\W$ be independent distributions over $\R^d$ with second moments $\Sigma_X=\E_x[xx^\top], \Sigma_W=\E_w[ww^\top]$. If we decompose 
$$ \Sigma_X = Q_X^\top Q_X, \ \Sigma_W = Q_W^\top Q_W, Q_XQ_W^\top = UDV^\top $$
with $UDV^\top$ the singular value decomposition of $Q_XQ_W^\top $, the minimizer of $\Phi(A)$ is 
$$  A^* = D^{1/2} U^\top Q_X^{-\top}$$

Letting $\sigma_X, \sigma_W$ be the vectors of the square roots of the eigenvalues of $\Sigma_X, \Sigma_W$. We have
$$ \Phi(I) = \Tr{\Sigma_X} \Tr{\Sigma_W} = \|\sigma_X\|^2 \|\sigma_W\|^2, \Phi(A^*) = \ip{\sigma_X, \sigma_W}^2 $$
\end{theorem}

The above theorem provides the optimal solution to the problem of minimizing $\Phi(A)$ but its computational cost may be too steep. Although it is solvable in polynomial time, or even near bi-linear time, in the input and output dimension, in an approximate version, we could have a scenario where the input dimension is quite large and we cannot afford to have multiple passes over it. For this scenario we provide an alternative technique that does not achieve the global minimum of $\Phi$ but admits a much simpler solution, and has guarantees that in many settings are sufficiently good. The idea in high level is to ignore all off-diagonal values of $\Sigma_X, \Sigma_W$ and solve the problem assuming they are zero. A very similar idea has proven itself in the field of optimization \cite{duchi2011adaptive}, where the expensive step of normalizing via the covariance matrix is replaced with the analog step w.r.t the diagonal. Collecting those stats can easily be done using a single pass and the decomposition becomes trivial.

\begin{theorem} \label{thm:fast}
Let $\X,\W$ be distributions over $\R^d$ with second moments $\Sigma_X=\E_x[xx^\top], \Sigma_W=\E_w[ww^\top]$. If we restrict to diagonal matrices to preprocess, then we can minimize $\Phi$ with the following: Let $d_X,  d_W$ element-wise square root of the diagonals of $\Sigma_X, \Sigma_W$ and let $\hat{A}$ be the diagonal matrix whose $i$'th entry is\footnote{If the diagonal has a zero value, it means we can ignore that entry in the original data. Hence, we assume w.l.o.g that all diagonals are strictly positive} $d_X(i)^{-1/2} d_W(i)^{1/2}$. It holds that
$$ \Phi(I) = \Tr{\Sigma_X} \Tr{\Sigma_W} =\|d_X\|^2 \|d_W\|^2, \Phi(\hat{A}) = \ip{d_X, d_W}^2 $$
\end{theorem}

The above theorem, coupled with the Cauchy-Schwartz inequality shows that the diagonal approach of $\hat{A}$ can only be better than taking the identity matrix, i.e.\ using an oblivious random projection. Although a pathological case can be constructed in which $\Phi(A^*) \ll \Phi(\hat{A}) = \Phi(I)$, in the sections below we experiment with real data and see that this is rarely the case, meaning that there is a significant gap between $\Phi(I)$ and $\Phi(\hat{A})$. The proofs of Theorems~\ref{thm:full} and~\ref{thm:fast} are given in Appendix~\ref{sec:proof_ddrp}

% \paragraph{Approximate covariance matrix} In many cases we cannot access $\Sigma_X,\Sigma_W$ but only approximate versions of them. By normalizing the estimate towards the identity we can maintain the guarantee of $\Phi(A) \leq \Phi(I)$ in both the full covariance and diagonal setting. Since this is a technical issue we do not discuss it further in this paper. 

%%%%%%%%%%%%%%%%%%%%%%%%
%%%%%%%%%%%%%%%%%%%%%%%%
%%%%%%%%%%%%%%%%%%%%%%%%
%
% 
%
%%%%%%%%%%%%%%%%%%%%%%%%
%%%%%%%%%%%%%%%%%%%%%%%%
%%%%%%%%%%%%%%%%%%%%%%%%
\vspace{-0.05in}
\section{Applications}
\vspace{-0.05in}

In this section we show how the developed tools can be used to speed up approximate matrix multiplications and improve the quality of linear regression or classification.

\vspace{-0.05in}
\subsection{Fast Matrix Multiplication} \label{sec:FMM}
\vspace{-0.05in}
One natural application in which we want to compress data from two different distributions arises in fast matrix multiplication (FMM). In this context, given two matrices $X, W$, we have the $i,j$th entry of $X^\top W$ is $\ip{x_i, w_j}$ where $x_i$ is the $i$th column of $X$ and $w_j$ the $j$th column of $W$. It follows that in order to compress the matrices for FMM it is sensible to compress their columns. We get the following simple re-scaling algorithm for FMM presented in Algorithm~\ref{alg:fmm}. Despite the simplicity of this variance scaling trick, in practice we see notably decreases in mean squared error from unscaled random projections on a variety of datasets. Details are in \S\ref{sec:experiments}.

\begin{algorithm}[tb]
  \caption{Fast Variance Scaling FMM} \label{alg:fmm}

\begin{algorithmic}
 \STATE {\bfseries Input:} Two matrice $X$, $W$
 \STATE $D_X$ diagonal matrix with $(D_X)_{j,j} \gets \E_i \diag{ X_{i,j}^2 }$\;
 \STATE $D_W$ diagonal matrix with $(D_W)_{j,j} \gets \E_i \diag{ W_{i,j}^2 }$\;
 \STATE $\tilde{X} \gets D_X^{-1/4}D_W^{1/4}X$\;
 \STATE $\tilde{W} \gets D_W^{-1/4}D_X^{1/4}W$\;
 \STATE With Random Projection $P$, project the columns of $\tilde{X}, \tilde{W}$ to $P \tilde{X}, P \tilde{W}$\;
 \STATE \textbf{Output: } $\tilde{X}^\top P^\top P \tilde{W}$\;
    \end{algorithmic}

 \end{algorithm}

\subsection{Linear regression and classification} \label{sec:linear}

Commonly in linear learning, either for regression or classification, the input dimension is quite large, possibly larger than the number of examples. A common approach for handling such cases, in order to mitigate both the danger of over-fitting and the large run-time, is to apply a random projection to the data and solve the regression problem on the lower dimension projected data. We note that these techniques are somewhat different than regularization based techniques, or methods aiming to find a sparse regressor. The advantage of this method has been established in previous works, and in a nutshell, comes from both having to learn a small number of parameters to begin with, hence obtaining faster run-time, and dealing with settings where the regressor is not necessarily sparse. 

The analysis of this approach follows from observing that for a random projection $R$, the optimal regressor $w^*$ and any data point $x$ we have 
$$ \ip{w^*,x} \approx \ip{Rw^*, Rx} .$$
It follows that by solving the problem on the projected data, our loss is upper bounded by the loss of $Rw^*$, which in turn is bounded due to the approximation guarantees of the random projection.

We cannot apply the asymmetric approach naively as we do not have access to the distribution of $w^*$. That being said, in most solutions to the problem one typically assumes an isotropic prior (translating to Ridge regression), meaning that $\Sigma_W=\lambda I$ for some scalar $\lambda$. Taking this approach exactly dictates that we pre-process the inputs $x$ by multiplying them by $\Sigma_X^{-1/4}$, or taking the more practical approach, by the diagonal matrix $D_X^{-1/4}$ where $(D_X)_{i,i}$ is the expected value of $x_i^2$. 

This approach however, depends too heavily on the prior assumption of $w^*$ that may not be correct. Taking this in to account we consider a more flexible approach by adding a hyper-parameter $\lambda$ and performing a pre-processing step of multiplying the data by $D_X^{\lambda}$. Setting $\lambda=-0.25$ recovers the above approach, and $\lambda=0$ recovers the approach of oblivious random projections. For the optimization procedure, it is possible to treat $\lambda$ as a hyper parameter and use a solver for the linear learning problem. Another option is to use any gradient based solver on the joint space of the low dimensional regressor and $\lambda$. Specifically, we draw a fixed random projection $R:\R^d \to \R^k$ mapping the input of dimension $d$ into a $k$-dimensional space with $k \ll d$, then minimize
$$ \min_{w,\lambda} \sum_{i=1}^n L(w^\top R\cdot D_X^{\lambda} \cdot x_i, y_i) $$
Here $x_i,y_i$ are the $i$'th datapoint and label, $R,D_X$ are fixed as detailed above, $w\in \R^k, \lambda \in \R$ are the parameters to be optimized, and $L$ is the loss function (e.g. logistic loss). 

% In the context of linear learning, we view our data as drawn from a distribution $\X$ and our regressor as being drawn from a distribution $\W$. Although we do not have access to the distribution $\W$ in most solutions to the problem one typically assumes an isotropic prior (translating to Ridge regression), meaning that $\Sigma_W=I$. Keeping that assumption, the theorems above dictate that we should preprocess our data with $\Sigma_X^{-1/4}$. 

% While this provides a rigorous guarantee for least square regression in the setting where the prior is known, we acknowledge that in practice the prior assumption is too far from the truth. To correct for this we treat our rescaling of our covariance as a hyperparameter by scaling our data with $\Sigma_X^{\lambda}$. This is equaivalent to assuming that the regressor $w$ has a prior with a covariant of $\Sigma_W = \Sigma_X^{\beta}$ for some scalar $\beta$. Although this assumption hardly captures all possible priors for the regressor we provide experiments showing that for properly chosen $\lambda$ it is possible to achieve results superior to those obtained in the oblivious random projection setting.

Our experiments in section~\ref{sec:experiments} show that a suitable value for this $\lambda$ parameter can significantly improve the performance of regression and classification tasks. Prior to this work, we are aware of two commonly used values of $\lambda$. Oblivious random projections correspond to fixing $\lambda=0$. Applying the random projection on normalized data corresponds to setting $\lambda=-0.5$. Our experiments show that often, a third, different value for $\lambda$ is far better than these two options, demonstrating the effectiveness of this approach.

\section{Experiments}\label{sec:experiments}
We proceed to experiments with real and synthetic data. In all of our experiments, in order to reduce the noise coming from the randomness of random projections, we calculate empirical mean of 100 trials. Throughout we use a random projection matrix of i.i.d.\ signs \cite{achlioptas2003database}. 
% To test the preprocessing procedures, we calculate empirical mean squared error using 100 trials. 

\subsection{Fast Matrix Multiplication}

For FMM we created a collection of $X,W$ pairs obtained either from synthetic data or real world public datasets. Due to space restrictions we defer the experiments on synthetic data to Appendix~\ref{app:more experiments}.

For real data matrices we consider two dense and two sparse dataset obtained from UCI Machine Learning Repository \cite{uci-arcene, uci-isolet, uci-repeat}. The first dense dataset, {\bf ARCENE}, was obtained by merging three mass-spectrometry datasets which indicate the abundance of proteins in human sera having a given mass value. This dataset has 10000 features and 700 data points. The second dense dataset, {\bf Isolet}, consists of 1559 data points with 616 features. In this dataset, 150 subjects spoke the name of each letter of the alphabet twice. The features include spectral coefficients; contour features, sonorant features, pre-sonorant features, and post-sonorant features. 
The first sparse dataset, {\bf TW\_OC}, is a dataset consisting of tweets with geolocation from Orange County CA area. The data coordinates correspond to the number of times a user visits a certain place and has 5000 data points with 11346 features. The second dataset, {\bf Go\_SF}, is similarly formatted and consists of check-ins from the app Gowalla, from the San Francisco area. It has 2593 data points with 7706 features.

We convert each of these four datasets into two $X,W$ pairs by splitting their feature sets in half. Thus if we have $n$ datapoints and $d$  features we end up with two matrices of size $n \times d/2$. We then either compare $X, W$ or $X^\top, W^\top$. This corresponds to either compressing to compare the inner products of the data points represented by disjoint features, or to comparing half of the feature vectors against the other half. We refer to the these as {\bf data} and {\bf feature}. We end up with eight matrix pair corresponding to each (dataset, data/feature) tuple. An example to make things concrete: The {\bf ARCENE-data} matrix pair refers to the $X,W$ pair (rather than $X^\top,W^\top$) obtained from the {\bf ARCENE} dataset.

% \subsection{Experiment}

For each of the $X,W$ pairs we compute the exact product $XW^\top$, and 3 approximate matrix products $(XR)(WR)^\top$ corresponding to three random projections $R$. The first, {\bf oblivious} is an oblivious random projection with i.i.d.\ signs. The second {\bf quick}, contains a pre-processing component based only on the diagonal entries of the covariance of $X$ and $W$.   The third, {\bf optimal}, contains the optimal pre-processing component based on CCA. 

For each approximate matrix product we compute the Squared Error, namely $\|XW^\top - (XR)(WR)^\top\|_F^2$ for multiple values of target dimension. We repeated the experiment 100 times to account for the noise coming from the randomness of $R$.

\paragraph{Dense Data}
For the dense data, we proceed as detailed with the two sets \textbf{ARCENE} and \textbf{Isolet} and corresponding four matrix pairs. Both in the case of {\bf ARCENE-feature}, and {\bf Isolet-feature} all 3 methods are nearly identical, hence we do not report the exact numbers of the experiment. This occurs since the corresponding covariance matrices are very similar. For {\bf ARCENE-data} and {\bf Isolet-data} the 3 methods provide different results. For the \textbf{ARCENE-data} pair, quick random projections yield a 2.33 times decrease in MSE, while optimal projections yield a 157.9x decrease in MSE. For the \textbf{Isolet-data} pair, quick random projections yields a 1.123 times decrease in MSE, and optimal projections yield 2.47 times decrease in MSE. The plot of MSE compared to target dimension is given in Figure~\ref{fig:fmm_biological}. For better visibility, the $y$-axis is the log of the MSE. The dotted lines represent one standard deviation. Recalling that the numbers reported are the mean of 100 trials, we expect a Guassian-like distribution of the measurement error.

% In both situations when comparing the transpose problem, which compares the feature vectors against one another, we get negligible results. On the contrary, when compressing the data points represented by different feature vectors, our preprocessing yields significant improvements. For the \textbf{ARCENE} data set, quick random projections yield a 2.33 times decrease in MSE, while optimal projections yield a 157.9x decrease in MSE. For the \textbf{Isolet} data set, quick random projections yields a 1.123 times decrease in MSE, and optimal projections yield 2.47 times decrease in MSE. The plot of MSE compared to target dimension is given in Figure~\ref{fig:fmm_biological}

\begin{figure}[ht]
  \vspace{-0.1in}
\begin{tabular}{cc}
\includegraphics[width=.4\linewidth]{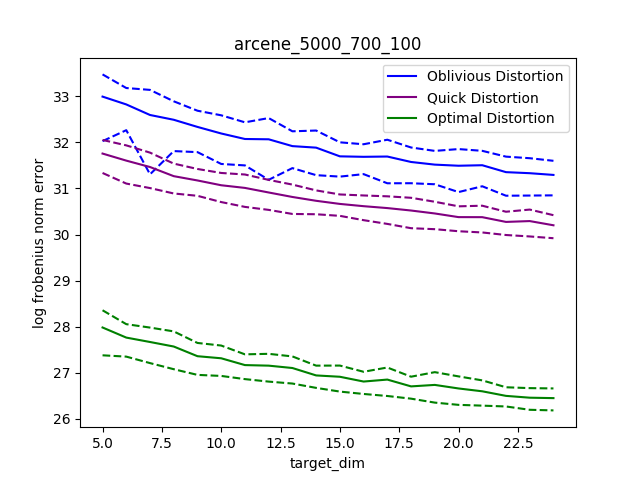}  &
  \includegraphics[width=.4\linewidth]{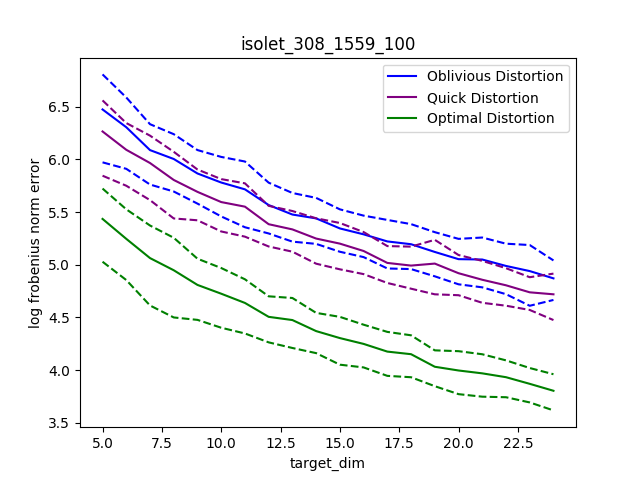}\\
  \end{tabular}
    \vspace{-0.1in}
  \caption{Dense Data FMM. Left \textbf{ARCENE-data}. Right \textbf{Isolet-data}. X-axis is the target dimension, Y-axis is the log MSE, dotted line represent lower and upper confidence bounds according to a single standard deviation.
  }

  \label{fig:fmm_biological}
\end{figure}

\paragraph{Sparse Data} 
For sparse datasets we see a significant advantage to our methods in both the {\bf data} and {\bf feature} matrix pairs. For {\bf Go\_SF-data}, quick projections yield 2\% of the MSE of oblivious projections, while optimal projections yield .9\% of the MSE of oblivious projections. For {\bf Go\_SF-feature}, quick projections yield 50.1\% of the MSE of oblivious projections and optimal projections give 41.4\% of oblivious projections. For both {\bf TW\_OC-data} and {\bf TW\_OC-feature} the quick and optimal distortions are nearly indistinguishable. For  {\bf TW\_OC-data} we get .08\% of the MSE of oblivious projections and for  {\bf TW\_OC-feature} the MSE is 2.2\% of the MSE of oblivious projections. The plots are given in Figure~\ref{fig:fmm_sparse} in an analog format to Figure~\ref{fig:fmm_biological}.

% preprocessing is great on both the data points and the feature vectors. For {\bf Go\_SF}, when looking at projections of the data points, quick projections yield 2\% of the MSE of oblivious projections, while optimal projections yield .9\% of the MSE of oblivious projections. For the transpose problem on the feature vectors, quick projections yield 50.1\% of the MSE of oblivious projections and optimal projections give 41.4\% of oblivious projections. For {\bf TW\_OC}, in both cases quick and optimal distortions are basically indistinguishable. In the case of projecting data points, we get .08\% of the MSE of oblivious projections. For feature vectors both have MSE which is 2.2\% of the MSE of oblivious projections. The plot of MSE compared with target dimension is given in Figure~\ref{fig:fmm_sparse}.

\begin{figure}[ht]
% \vspace{-0.3in}

\begin{tabular}{cc}
  \includegraphics[width=.4\linewidth]{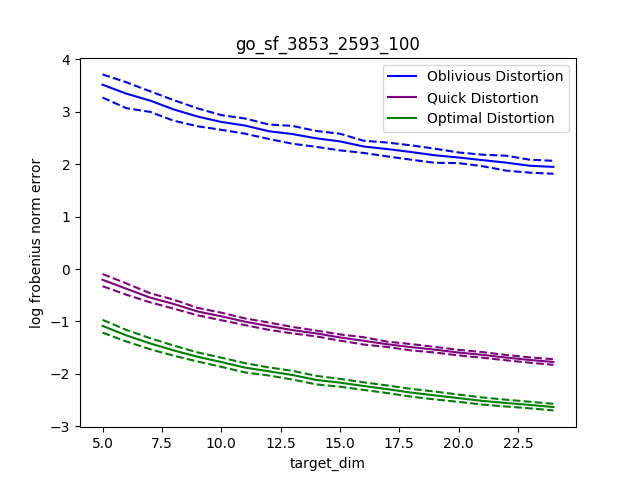} &
  \includegraphics[width=.4\linewidth]{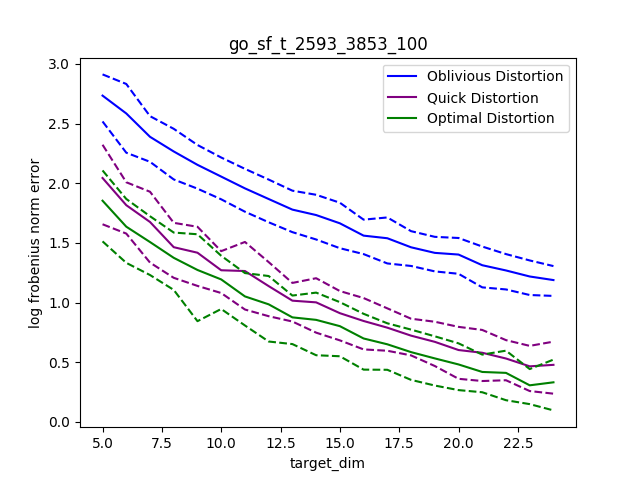} \\
  \includegraphics[width=.4\linewidth]{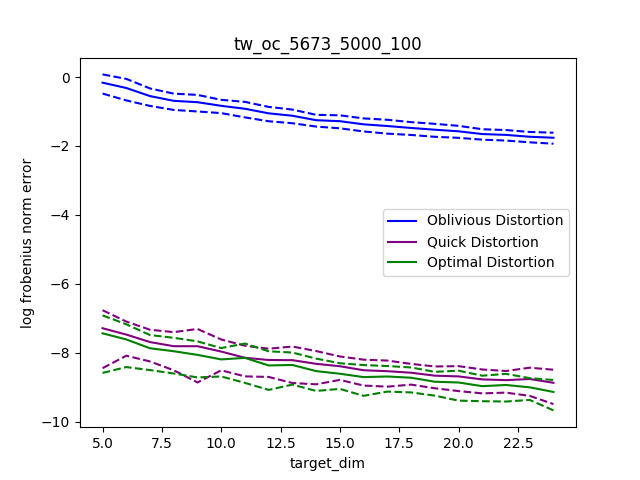}&
  \includegraphics[width=.4\linewidth]{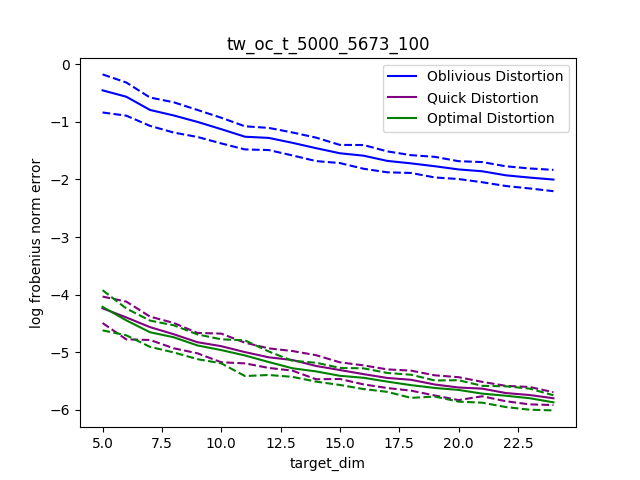}\\
  \end{tabular}  \caption{Sparse Data FMM. Top left is {\bf Go\_SF-data}. Top right is {\bf Go\_SF-feature}. Bottom left is {\bf TW\_OC-data}. Bottom right is {\bf TW\_OC-feature}. X-axis is the target dimension, Y-axis is the log MSE, dotted line represent lower and upper confidence bounds according to a single standard deviation.}
  \label{fig:fmm_sparse}
\end{figure}

\subsection{Regression}

\subsubsection{Linear Regression}\label{sec:reg_linear}
We used two data sets from the UCI Machine Learning Repository \cite{uci-e2006, uci-slice}. {\bf Slice Localization} was retrieved from a set of 53500 CT images from 74 different patients. The feature vector consists of two histograms in polar space describing the bone structure and the air inclusions inside the body; it has 54500 samples with 384 (dense) features. The {\bf E2006} dataset consists of 10K reports from thousands of publicly traded U.S. companies, published in 1996–2006 and stock return volatility measurements in the twelve-month period before and the twelve-month period after each report. The data is encoded using term frequency-inverse document frequency, resulting in a sparse dataset with 16087 samples and 150630 features. 

In our experiment we apply the projection with a preprocessing step matching a scalar $\lambda$, as described above, to the datasets. Once the data is projected we solve the regression problem on the low dimensional space.  Table~\ref{tab:reg_linear} shows the mean square error (MSE) $\pm$ one standard deviation for with different values of $\lambda$ and projection target dimension $k$. For the {\bf Slice Localization} dataset we observe near optimal performance at $\lambda = -.75$, giving significantly better results than the classic random projection corresponding to $\lambda=0$, cutting the MSE by a factor of more than 2. Even compared to $\lambda=-0.5$ corresponding to the strategy of normalizing the data before applying the RP, for some values of $k$ there is a statistically significant advantage for the best $\lambda$ value, see e.g.\ $k=25$.

For {\bf E2006} the optimal value for $\lambda$ is interestingly around $1.75$, giving empirical justification to setting $\lambda$ as a (possibly trainable) parameter rather than fixing it as a constant. For negative values of $\lambda$ the results were worse than those of positive $\lambda$ values and are not reported. The improvement over RPs ($\lambda=0$) is mild compared to the {\bf Slice} datasets, giving e.g.\ 1.3\% reduction in the MSE for $k=250$ as opposed to cutting it by half. Nevertheless, considering the standard deviation the improvement remains statistically significant.

% Figure~\ref{fig:reg_linear} shows the result of our experiment when applied with different values of $\lambda$ and for a different target dimension for the projection. For the {\bf Slice Localization} dataset we observe near optimal performance at $\lambda = -.25$, and see significant improvements from the raw regression. For $k=20$, oblivious projection $\lambda=-0.25$ yielded 57.8\% of the MSE of oblivious projections, while $\lambda=-.5$ achieved 38.8\%. For {\bf E2006} we in fact see negligible change, for $k=200$, $\lambda=1.75$ achieves a 1\% decrease in MSE from oblivious projections.

\begin{table}[ht]
    \centering
    \tiny
    \begin{tabular}{|c|c|c|c|c|c|c|c|c|c|}
        \hline 
         \multicolumn{10}{|c|}{\bf{E2006}}        \\
        \hline
        $\lambda$ & 0.0 & 0.25 & 0.5 & 0.75 & 1.0 & 1.25 & 1.5 & 1.75 & 2.0\\
\hline
\hline
$k=50$ & 266.5  & 265.4  & 265.5  & 265.5  & {\bf 265.3 } & 265.5  & 265.5  & 265.6  & 265.5 \\
 & $\pm$ 1.0 & $\pm$ 0.3 & $\pm$ 0.3 & $\pm$ 0.3 & {\bf $\pm$ 0.4 } & $\pm$ 0.2 & $\pm$ 0.3 & $\pm$ 0.3 & $\pm$ 0.3\\
\hline
$k=100$ & 265.7  & 265.1  & 265.0  & 264.6  & 264.6  & 264.7  & {\bf 264.6 } & 264.9  & 264.7 \\
 & $\pm$ 0.5 & $\pm$ 0.4 & $\pm$ 0.4 & $\pm$ 0.4 & $\pm$ 0.6 & $\pm$ 0.4 & {\bf $\pm$ 0.6 } & $\pm$ 0.6 & $\pm$ 0.6\\
\hline
$k=150$ & 265.4  & 264.6  & 264.0  & 263.8  & 264.1  & 263.7  & 263.4  & 263.5  & {\bf 263.1 }\\
 & $\pm$ 0.6 & $\pm$ 0.6 & $\pm$ 0.5 & $\pm$ 0.6 & $\pm$ 0.6 & $\pm$ 0.5 & $\pm$ 0.6 & $\pm$ 0.7 & {\bf $\pm$ 0.5 }\\
\hline
$k=200$ & 265.3  & 264.4  & 263.7  & 263.3  & 263.2  & 263.1  & 262.9  & 262.7  & {\bf 262.6 }\\
 & $\pm$ 0.7 & $\pm$ 0.6 & $\pm$ 0.7 & $\pm$ 0.5 & $\pm$ 0.7 & $\pm$ 0.6 & $\pm$ 0.6 & $\pm$ 0.7 & {\bf $\pm$ 0.5 }\\
\hline
$k=250$ & 265.2  & 264.0  & 263.7  & 262.8  & 262.6  & 262.4  & 262.3  & {\bf 261.8 } & 261.9 \\
 & $\pm$ 0.7 & $\pm$ 0.5 & $\pm$ 0.9 & $\pm$ 0.7 & $\pm$ 0.7 & $\pm$ 0.6 & $\pm$ 0.6 & {\bf $\pm$ 0.5 } & $\pm$ 0.5\\
\hline
\hline
\multicolumn{10}{|c|}{\bf{Slice ($\times 10^5$)}}        \\
        \hline
$\lambda$ & -2.0 & -1.75 & -1.5 & -1.25 & -1.0 & -0.75 & -0.5 & -0.25 & 0.0\\
\hline
\hline
$k=5$ & 23.9  & 24.5  & 21.8  & 22.9  & 23.2  & {\bf 21.4 } & 22.9  & 39.8  & 66.1 \\
 & $\pm$ 1.0 & $\pm$ 0.9 & $\pm$ 1.0 & $\pm$ 1.4 & $\pm$ 1.7 & {\bf $\pm$ 1.0 } & $\pm$ 2.5 & $\pm$ 7.5 & $\pm$ 18.6\\
\hline
$k=10$ & 21.3  & 20.7  & 19.6  & 19.9  & {\bf 18.7 } & 19.2  & 18.9  & 24.8  & 45.4 \\
 & $\pm$ 1.3 & $\pm$ 1.2 & $\pm$ 1.2 & $\pm$ 1.4 & {\bf $\pm$ 2.1 } & $\pm$ 1.6 & $\pm$ 2.8 & $\pm$ 4.9 & $\pm$ 8.7\\
\hline
$k=15$ & 18.9  & 19.2  & 18.4  & 17.6  & 16.3  & 15.5  & {\bf 13.5 } & 25.1  & 38.7 \\
 & $\pm$ 0.9 & $\pm$ 0.8 & $\pm$ 0.7 & $\pm$ 1.0 & $\pm$ 1.5 & $\pm$ 1.9 & {\bf $\pm$ 1.0 } & $\pm$ 2.1 & $\pm$ 4.2\\
\hline
$k=20$ & 18.1  & 17.1  & 17.5  & 17.2  & 15.5  & 14.1  & {\bf 13.1 } & 19.5  & 33.6 \\
 & $\pm$ 0.5 & $\pm$ 0.2 & $\pm$ 0.6 & $\pm$ 0.7 & $\pm$ 1.2 & $\pm$ 1.9 & {\bf $\pm$ 1.2 } & $\pm$ 3.8 & $\pm$ 7.9\\
\hline
$k=25$ & 17.0  & 17.0  & 16.2  & 15.0  & 12.9  & {\bf 11.5 } & 14.0  & 15.5  & 29.7 \\
 & $\pm$ 0.3 & $\pm$ 0.3 & $\pm$ 1.1 & $\pm$ 0.6 & $\pm$ 0.7 & {\bf $\pm$ 0.5 } & $\pm$ 1.2 & $\pm$ 1.7 & $\pm$ 7.8\\
\hline
    \end{tabular}
    \caption{Linear Regression. Detailed in \S \ref{sec:reg_linear}. Results contain the Mean Square Error on of the trained model for different target dimensions, denoted by $k$, and $\lambda$ values, for the \textbf{E2006} and \textbf{Slice} datasets}
    \label{tab:reg_linear}
    \vspace{-0.15in}
\end{table}

% \begin{figure}[ht]
% \begin{tabular}{cc}
%   \includegraphics[width=.4\linewidth]{"experiments/regression/slice_localization_384_53500_5_lamb_5_0".png} &
%   \includegraphics[width=.4\linewidth]{"experiments/regression/e2006_full_150360_16087_20_lamb_50_6".png} 
% %   \includegraphics[width=.3\linewidth]{"experiments/regression/e2006_full_150360_16087_20_lamb_50_5".png} \\
%   \end{tabular}
  
%   \caption{Linear Regression. Detailed in \S \ref{sec:reg_linear}. To the left is \textbf{slice\_localization}. To the right is \textbf{e2006}.
%   }
% %   Plotted empirical mean linear loss against hyper-parameter lambda. Slice localization dataset has 53500 samples with 384 features gathered from biological sources. We see the best results with $\lambda \sim -.50$. E2006 dataset is sparse data derived from term frequency–inverse document frequency (16087 samples, 150630 dimensional). Very little improvement for positive $\lambda$ and large blow up for negative values of $\lambda$. The isotropic prior causes huge blow up for the sparse data, but seems near optimal for the biological data.
%  \label{fig:reg_linear}
% \end{figure}

\subsubsection{Logistic Regression}\label{sec:logistic}
Here we used two datasets with a classification task of logistic regression \cite{data-cifar, data-rcv1}. The first dataset if {\bf Cifar10}, a well known dataset used for image classification. It consists of $32 \times 32$ color images, resulting in a 3072 dimensional data space. We take 1979 samples and project them each 100 times. The second is Reuters Corpus Volume I ({\bf RCV1}), an archive of over 800,000 manually categorized newswire stories recently made available by Reuters, Ltd.\ for research purposes. In our experiments we used 20242 stories each encoded using term frequency-inverse resulting in 47236 dimensions. 

Our experiments are the same as in the linear regression case. We project each dataset to several dimensions, with several preprocessing steps corresponding to different $\lambda$ scalars. The results are given in Table~\ref{tab:reg_logistic}, where for every target dimension $k$ and $\lambda$ value we report the accuracy score $\pm$ a single standard deviation. For {\bf Cifar10} the result is not very sensitive to the value of $\lambda$ as long as it's in the range of $[-1,0]$; tuning $\lambda$ does not provide a statistically significant improvement over the oblivious RP. For {\bf RCV1} however we see a clear gain of tuning $\lambda=0.25$ providing, e.g.\ for $k=200$ a lift of $4\%$ in accuracy compared to oblivious RPs. 

As with linear regression we see that different datasets and tasks correspond to different optimal values of $\lambda$ justifying our proposal to tune it as a part of the learning process.

\begin{table}[ht]
    \centering
    \tiny
    \begin{tabular}{|c|c|c|c|c|c|c|c|c|c|c|c|c|c|}
        \hline 
        $\lambda$ & -1.5  & -1.25 & -1.  & -0.75  &  -0.5 & -0.25 & 0. &   0.25&  0.5 &  0.75&  1. &   1.25&  1.5 \\ 
        \hline 
        \hline 
         \multicolumn{14}{|c|}{\bf{Cifar10}}        \\
        \hline
        $k=50$ & 49.5  & 67.2  & 68.5  & 67.9  & 67.9  & {\bf 69.3 } & 68.2  & 64.2  & 64.3  & 63.7  & 65.2  & 64.5  & 63.6 \\
 & $\pm$ 2.5 & $\pm$ 4.4 & $\pm$ 3.9 & $\pm$ 4.6 & $\pm$ 4.0 & {\bf $\pm$ 3.3 } & $\pm$ 4.2 & $\pm$ 4.8 & $\pm$ 5.0 & $\pm$ 5.0 & $\pm$ 4.9 & $\pm$ 4.6 & $\pm$ 5.0\\
\hline
$k=100$ & 48.7  & 71.8  & 72.2  & 72.0  & 71.8  & 71.9  & {\bf 72.5 } & 68.5  & 68.6  & 68.5  & 69.5  & 68.5  & 68.8 \\
 & $\pm$ 1.1 & $\pm$ 2.9 & $\pm$ 3.2 & $\pm$ 3.4 & $\pm$ 3.4 & $\pm$ 3.2 & {\bf $\pm$ 3.2 } & $\pm$ 4.2 & $\pm$ 3.9 & $\pm$ 3.6 & $\pm$ 3.7 & $\pm$ 4.2 & $\pm$ 3.7\\
\hline
$k=150$ & 48.6  & 73.0  & 74.0  & 74.1  & 73.9  & {\bf 74.5 } & 74.5  & 71.3  & 71.1  & 71.8  & 71.7  & 71.6  & 71.1 \\
 & $\pm$ 0.9 & $\pm$ 2.5 & $\pm$ 2.9 & $\pm$ 2.7 & $\pm$ 2.9 & {\bf $\pm$ 2.4 } & $\pm$ 2.9 & $\pm$ 3.5 & $\pm$ 3.7 & $\pm$ 3.4 & $\pm$ 3.4 & $\pm$ 3.3 & $\pm$ 3.3\\
\hline
$k=200$ & 48.4  & 73.8  & 75.3  & 75.9  & 75.2  & 76.0  & {\bf 76.0 } & 73.3  & 73.1  & 73.5  & 73.4  & 73.8  & 73.6 \\
 & $\pm$ 0.4 & $\pm$ 2.2 & $\pm$ 2.4 & $\pm$ 2.2 & $\pm$ 2.3 & $\pm$ 2.5 & {\bf $\pm$ 2.6 } & $\pm$ 2.8 & $\pm$ 3.0 & $\pm$ 3.1 & $\pm$ 3.0 & $\pm$ 2.8 & $\pm$ 3.4\\
\hline
$k=250$ & 48.4  & 74.4  & {\bf 76.7 } & 76.6  & 76.4  & 76.3  & 76.6  & 74.3  & 74.8  & 74.8  & 74.9  & 74.9  & 75.4 \\
 & $\pm$ 0.4 & $\pm$ 1.8 & {\bf $\pm$ 2.5 } & $\pm$ 2.3 & $\pm$ 2.3 & $\pm$ 2.1 & $\pm$ 2.6 & $\pm$ 2.8 & $\pm$ 3.0 & $\pm$ 3.2 & $\pm$ 2.8 & $\pm$ 2.8 & $\pm$ 2.4\\
\hline
         \multicolumn{14}{|c|}{\bf{RCV1}}        \\
        \hline
$k=50$ & 50.2  & 50.4  & 49.9  & 56.6  & 56.0  & 55.4  & 56.7  & 57.9  & {\bf 58.6 } & 56.9  & 56.9  & 56.9  & 56.9 \\
 & $\pm$ 1.5 & $\pm$ 1.6 & $\pm$ 1.6 & $\pm$ 0.7 & $\pm$ 0.9 & $\pm$ 1.3 & $\pm$ 1.8 & $\pm$ 2.0 & {\bf $\pm$ 1.6 } & $\pm$ 0.0 & $\pm$ 0.1 & $\pm$ 0.1 & $\pm$ 0.0\\
\hline
$k=100$ & 49.8  & 50.0  & 50.2  & 57.1  & 56.6  & 57.2  & 58.2  & {\bf 61.7 } & 60.9  & 56.9  & 56.9  & 56.9  & 56.9 \\
 & $\pm$ 1.0 & $\pm$ 1.4 & $\pm$ 1.3 & $\pm$ 0.8 & $\pm$ 0.8 & $\pm$ 1.3 & $\pm$ 2.0 & {\bf $\pm$ 2.7 } & $\pm$ 2.5 & $\pm$ 0.0 & $\pm$ 0.1 & $\pm$ 0.1 & $\pm$ 0.1\\
\hline
$k=150$ & 50.3  & 50.1  & 50.8  & 56.5  & 56.2  & 58.4  & 61.2  & {\bf 62.6 } & 60.8  & 56.9  & 56.9  & 56.9  & 56.9 \\
 & $\pm$ 1.4 & $\pm$ 1.1 & $\pm$ 1.1 & $\pm$ 1.0 & $\pm$ 0.8 & $\pm$ 1.5 & $\pm$ 1.9 & {\bf $\pm$ 2.1 } & $\pm$ 2.0 & $\pm$ 0.1 & $\pm$ 0.0 & $\pm$ 0.0 & $\pm$ 0.0\\
\hline
$k=200$ & 50.6  & 51.0  & 50.5  & 56.7  & 56.7  & 59.1  & 61.2  & {\bf 65.2 } & 60.9  & 56.9  & 56.9  & 56.9  & 56.9 \\
 & $\pm$ 1.2 & $\pm$ 1.7 & $\pm$ 1.1 & $\pm$ 0.9 & $\pm$ 1.1 & $\pm$ 1.6 & $\pm$ 1.2 & {\bf $\pm$ 1.6 } & $\pm$ 2.0 & $\pm$ 0.0 & $\pm$ 0.0 & $\pm$ 0.0 & $\pm$ 0.1\\
\hline
$k=250$ & 50.2  & 50.6  & 51.2  & 56.9  & 57.4  & 59.0  & 62.3  & {\bf 65.9 } & 60.8  & 56.9  & 56.9  & 56.9  & 56.9 \\
 & $\pm$ 1.5 & $\pm$ 1.3 & $\pm$ 1.6 & $\pm$ 1.0 & $\pm$ 1.2 & $\pm$ 1.7 & $\pm$ 2.3 & {\bf $\pm$ 1.2 } & $\pm$ 1.5 & $\pm$ 0.0 & $\pm$ 0.0 & $\pm$ 0.1 & $\pm$ 0.1\\
\hline
\hline
    \end{tabular}
    \caption{Logistic Regression. Detailed in \S \ref{sec:logistic}. Results contain the accuracy on of the trained model for different target dimensions, denoted by $k$, and $\lambda$ values, for the \textbf{Cifar10} and \textbf{RCV1} datasets}
    \label{tab:reg_logistic}
        \vspace{-0.25in}

\end{table}

% \begin{figure}
% \begin{tabular}{cc}
%   \includegraphics[width=.4\linewidth]{"experiments/regression/cifar_logistic_3072_1979_100_lamb_50_0".png} &
%   \includegraphics[width=.4\linewidth]{"experiments/regression/rcv1_47236_20242_15_lamb_50_0".png}  \\
%   \end{tabular}
%   \caption{Logistic Regression. Detailed in \S \ref{sec:logistic}. On the left is \textbf{cifar}. On the right is \textbf{rcv1}.
%   }

% %   plotted sklearn score (accuracy) against target dimension. Cifar dataset is image classification, 1979 samples with 3072 attributes projected 100 times. The accuracy is maximized around $\lambda \sim -.25$. RCV1 dataset is 20242 data points with 47236 feature and has accuracy maximized around $\lambda \sim .25$. In both cases the isotropic prior behaves well. 
%  \label{fig:reg_logistic}
% \end{figure}

\vspace{-0.05in}    
\section{Future Directions}
\vspace{-0.05in}
In this paper we explore the simplest first step to looking at data dependent unbiased random projections. In doing this we restrict to linear projections, where each of the outputs are independent. An interesting idea to explore is what can we achieve if the output dimensions are dependent? Can we obtain stronger results with a non-linear pre-processing step? Can we achieve stronger results with a non-linear projection? Other than improved guarantees, the motivation for these methods come from them being applicable to the symmetric setting where both distributions are the same; a setting where our techniques fall back to the standard random projections.

\newcommand{\etalchar}[1]{$^{#1}$}

\newpage
\appendix

\section{Proofs of Section~\ref{sec:ddrp}} \label{sec:proof_ddrp}

\begin{proof} [Proof of Lemma~\ref{lem:iid_rp}]
We first observe that the estimate is indeed unbiased
\begin{align*}
    \E[\ip{Rx, Rw}] &= \E[x^\top R^\top Rw] \\
    &= x^\top\E[R^\top R]w \\
    &= x^\top w \\
    &= \ip{x,w}
\end{align*}
We move on to compute its variance. To that end, we begin with the case of $k=1$ where the dimension reduction is with a $d$ dimensional vector $r$ and denote by $r_i$ its $i$'th element. We compute the second moment explicitly; to that end we denote by $s_{2,2} = \E_{r_i,r_j} r_i^2 r_j^2$, $s_{4} = \E_{r_i } r_i^4$. For our distribution we have $s_{2,2}=1, s_{4} \leq 3$.
\begin{align*}
    \E_{r } \ip{x, r}^2 \ip{w, r}^2 &= \E_r \left( \sum_i x_i r_i \right)^2 \left( \sum_i w_i r_i \right)^2\\
 &= \E_{r} \left(\sum_i x_i r_i\right)^2 \left(\sum_i w_i r_i\right)^2 \\
&=\E_r \left(\sum_i x_i^2 r_i^2 +  \sum_{i \neq j} x_i x_j r_i r_j \right) \left(\sum_i w_i^2 r_i^2 + 2 \sum_{i < j} w_i w_j r_i r_j \right)\\
&= \E_r \left(\sum_i x_i^2 w_i^2 r_i^4 + 2 \sum_{i \neq j} x_i x_j w_i w_j r_i^2 r_j^2 + \sum_{i \neq j} x_i^2 w_j^2 r_i^2 r_j^2\right) \\
&= s_4 \sum_i x_i^2 w_i^2  + 2 s_{2,2} \sum_{i \neq j} x_i x_j w_i w_j  + s_{2,2} \sum_{i \neq j} x_i^2 w_j^2 \\
&= s_4 \sum_i x_i^2 w_i^2  + 2 s_{2,2}\left(\ip{x,w}^2 - \sum_i x_i^2 w_i^2\right) + s_{2,2} \sum_{i \neq j} x_i^2 w_j^2 \\
&= s_4 \sum_i x_i^2 w_i^2  + 2 s_{2,2}\left(\ip{x,w}^2 - \sum_i x_i^2 w_i^2\right) + s_{2,2} \left( \|x\|^2\|w\|^2 - \sum_i x_i^2 w_i^2 \right)  \\
&= \left(s_4 - 3s_{2,2}\right)  \sum_i x_i^2 w_i^2 + 2s_{2,2} \ip{x, w}^2 + s_{2,2} \|x\|^2 \|w\|^2 \\
&\leq 2\ip{x, w}^2 + \|x\|^2 \|w\|^2 \\
\end{align*}
We get that the variance of the estimate of $\ip{x,w}$ is
$$\Var \leq \ip{x,w}^2 + \|x\|\|w\|^2$$
For the case of $R$ being a matrix with target dimension $k$, since everything is i.i.d, the estimate is simply an average of $k$ independent estimates of target dimension 1, and the variance is the same as above, divided by $k$.
\end{proof}

\begin{proof} [Proof of Theorem~\ref{thm:full}]
Let $Q_X^\top Q_X = \Sigma_X, Q_W^\top Q_W = \Sigma_W$. Using the independence of $x$ and $w$ we get
\begin{align*}
& \E_{x,w}\|Ax\|^2\|(A^{\top})^{-1}w\|^2 \\ 
&= \Tr{A^{\top} A \Sigma_X}\Tr{(A^{-1} (A^{-1})^{\top}\Sigma_W}\\
&= \Tr{Q_X A^{\top} A Q_X^\top}\Tr{(Q_W A^{-1} (A^{-1})^{\top}Q_W^\top}&\\
&= \|A Q_X^\top\|^2_F \|A^{-\top} Q_W^\top\|^2_F&
\end{align*}
By Cauchy-Schwarz on the Frobenius Inner Product we get the universal lower bound:
\begin{align*}
     \Tr{ Q_X Q_W^{\top} } &= \langle A Q_X^\top, A^{-\top} Q_W^\top \rangle_F^2\\
&\leq \|A Q_X^\top\|^2_F \|A^{-\top} Q_W^\top\|^2_F 
\end{align*}
We use the Frobenius Inner Product
$\ip{X,Y}_F = \Tr{X^\top Y}$ and the Cauchy-Schwartz inequality stating that $\ip{X,Y}_F \leq \|X\|_F \|Y\|_F$. It implies that
\begin{align*}
\Tr{ Q_X Q_W^\top }^2 &= \Tr{ Q_X A^\top A^{-\top} Q_W^\top }^2\\
&= \ip{ A Q_X^\top, A^{-\top} Q_W^\top }_F^2 \\
&\leq \|A Q_X^\top\|^2_F \|A^{-\top} Q_W^\top\|^2_F \\
&= \E_{x,w}\|Ax\|^2\|A^{-\top}w\|^2 \end{align*}

Notice that this inequality happens for any factorization matrices $Q_X,Q_W$ and any invertible matrix $A$. Furthermore, the lefthand size is independent of $A$ and the righthand side is independent of $Q_X,Q_W$. It follows that if we find a specific triplet $Q_X,Q_W,A$ such that
\[ 
\ip{ A Q_X^\top, A^{-\top} Q_W^\top }_F^2 = \|A Q_X^\top\|^2_F \|A^{-\top} Q_W^\top\|^2_F
\]
we get that $Q_X,Q_W$ maximize the left expression and $A$ minimizes the right expression. Now, for two matrices $X,Y$ it holds that $\ip{X,Y}_F =  \|X\|_F \|Y\|_F$ only if $X=\lambda Y$ for some scalar $\lambda$. It follows that w.l.o.g.\ our matrix $A$ is such that 
$$ A Q_X^\top = A^{-\top} Q_W^\top$$
or conversely
$$ A^{\top}A Q_X^\top = Q_W^\top $$
Such a matrix $A$ can be found via CCA
\begin{lemma}[CCA]
Given two positive definite matrices $\Sigma_X, \Sigma_W$, we can pick matrix $A$ such that 
\[A  \Sigma_X A^\top =  A^{-\top} \Sigma_W A^{-1}  \]
If we decompose $\Sigma_X = Q_X^\top Q_X$, $\Sigma_W = Q_W^\top Q_W$, $Q_X Q_W^\top = U_{XW} D_{XW} V_{XW}^\top$, then we can set
\[ A = D_{XW}^{1/2} U_{XW}^\top Q_X^{-\top} \]
\[ A^{-\top} = D_{XW}^{1/2}V_{XW}^\top Q_W^{-\top} \]
Furthermore this choice of $A$ is independent of the decomposition of $\Sigma_X$ or $\Sigma_W$
\end{lemma}

We choose $Q_X$ arbitrarily and set $Q_W$ in a way that $Q_XQ_W^\top$ is symmetric and positive definite. This can be done by choosing an arbitrary $\hat{Q}_W$, decomposing $Q_X\hat{Q}_W^\top = \hat{U}\hat{D}\hat{V}^\top $ and setting $Q_W = \hat{U} \hat{V}^\top \hat{Q}_W$. Since $\hat{V},\hat{U}$ are orthogonal matrices we still have $Q_W^\top Q_W = \Sigma_W$. We also get that
$$ Q_XQ_W^\top = Q_X\hat{Q}_W^\top \hat{V}\hat{U}^\top =  \hat{U}\hat{D}\hat{V}^\top \hat{V} \hat{U}^\top = \hat{U}\hat{D}\hat{U}^\top$$

\noindent Using the terms above we can plug in the equation for $A^*$, the optimizer of the CCA problem, and obtain
\begin{align*}(A^*)^\top A^* Q_X^\top &= Q_W^\top V_{XW}^{-\top} D_{XW}^{-1/2} D_{XW}^{1/2} U_{XW}^\top  Q_X^{-\top} Q_X^\top\\ & = Q_W^\top V_{XW}^{-\top} U_{XW}^\top  \\
\end{align*}
Since we chose $Q_X,Q_W$ in a way that $Q_XQ_W^\top$ is psd we get that $V_{XW}=U_{XW}=\hat{U}$ and 
$$(A^*)^\top A^* Q_X^\top =Q_W^\top $$
as required.

Now that we obtained the minimizer $A^*$ we can compute the value of the expression.  $\Tr{Q_XQ_W^\top}$ equals the sum of the eigenvalues of the matrix which are the same as its singular values since it is psd. That is, the value of $\Phi(A^*)$ is the sum of the elements of $D_{XW}$. Notice that these elements do not depend on our choice of $Q_X,Q_W$ as that choice only affects the rotations $V_{XW}, U_{XW}$. With that in mind we can compute the values of $D_{XW}$ by considering the decomposition $Q_X =  D_X^{1/2} U_X^\top$, $Q_W =  D_W^{1/2} U_W^\top$, where $\Sigma_X = U_X D_X U_X^\top, \ \ \Sigma_W = U_W D_W U_W^\top$ are the singular value decompositions of $\Sigma_X,\Sigma_W$, we get that these values are exactly those obtained by multiplying the square roots of the eigenvalues of $\Sigma_X, \Sigma_W$ as required.

\end{proof}

\begin{proof} [Proof of Theorem~\ref{thm:fast}]
In what follows we use the fact that the Frobenius inner product of $M$ with a diagonal matrix only depends on the diagonal entries of $M$. That is, if $X$ is diagonal then for any matrix $M$ we have $\Tr{XM}=\Tr{X\diag{M}}$ where $\diag{M}$ is the matrix that is equal to $M$ on the diagonal and zero elsewhere. Using the fact the $A$ is diagonal we get
\begin{align*}
& \|A Q_X^\top \|_F^2 \|A^{-\top} Q_W^\top \|_F^2\\
&= \Tr{A^\top A \Sigma_X} \Tr{A^{-1} A^{-\top} \Sigma_W} \\ 
&= \Tr{A^\top A \diag{\Sigma_X}} \Tr{A^{-1} A^{-\top} \diag{\Sigma_W}}  
\end{align*}
This simple calculation shows us that restricting to diagonal pre-processing is equivalent to throwing away all of the off-diagonal information in our covariance matrices, and proceeding with the results of Theorem~\ref{thm:full}. The claim trivialy follows.

\end{proof}

\section{Additional Experiments} \label{app:more experiments}

\subsection{FMM synthetic data}\label{sec:synth}

For synthetic matrices we define a few distributions over matrices. In the first, called {\bf diag}, we first sample a $d$ dimensional vector of i.i.d Laplace variables. This determines a diagonal $d \times d$ covariance matrix. Now we sample $n$ i.i.d.\ rows to construct the matrix. The second distribution, called {\bf uniform}, we sample $d$ i.i.d.\ uniform variables in $[0,1]$ as the eigenvalues of the covariance and a random rotation for the eigenvectors. With the covariance matrix ready we sample $n$ i.i.d.\ rows. The third type called {\bf unifskew} is obtained by averaging two independent matrices one drawn from {\bf diag} and one from {\bf uniform}.
The pairs of synthetic matrices we consider are {\bf diag-diag}, {\bf uniform-diag}, {\bf uniform-unifskew}, {\bf uniform-uniform}. Every pair of matrices $X,W$ consists of two independently drawn matrices from the mentioned distribution.

For each of the $X,W$ pairs we compute the exact product $XW^\top$, and 3 approximate matrix products $(XR)(WR)^\top$ corresponding to three random projections $R$. The first, {\bf oblivious} is an oblivious random projection with i.i.d.\ signs. The second {\bf quick}, contains a pre-processing component based only on the diagonal entries of the covariance of $X$ and $W$.   The third, {\bf optimal}, contains the optimal pre-processing component based on CCA. 
For each approximate matrix product we compute the Squared Error, namely $\|XW^\top - (XR)(WR)^\top\|_F^2$ for multiple values of target dimension. We repeated the experiment 100 times to account for the noise coming from the randomness of $R$.

For each of the three distributions detailed above, \textbf{diag}, \textbf{uniform}, \textbf{unifskew}, we draw 1000 samples from $\mathbb{R}^{100}$. We then form a matrices using these as columns, and look at the squared error over 100 random projections. The first experiment we conduct is looking at the mean squared error for two matrices, both with columns drawn from the \textbf{diag} distribution. For this experiment quick and optimal are equivalent (since the covariance matrix is diagonal) and nets around a 3x decrease in MSE, regardless of the target dimension. When comparing \textbf{diag} against \textbf{unif}, we use the same methodology. Here we see quick yields approximately a 1.5x decrease in MSE over oblivious, and optimal yields another 1.5x decrease in MSE from quick. When comparing \textbf{unif} against \textbf{unifskew} we get the same decreases as the previous experiement across all target dimensions. Finally we compare \textbf{unif} against \textbf{unif}. Here the quick projections yield the same MSE as oblivious projections, while optimal projections yield a 2x decrease in MSE.

These results demonstrate what we expect: when drawing from distributions whose covariance are mostly concentrated on the diagonal, quick random projections perform almost as well as optimal while requiring significantly less preprocessing time. In figure \ref{fig:fmm_synthetic} we plot the results of these experiments by plotting the mean squared error against different target dimensions. For each plot we also include dashed lines to show the standard deviation among the 100 random projections.

\begin{figure}[ht]
\begin{tabular}{cc}
  \includegraphics[width=.4\linewidth]{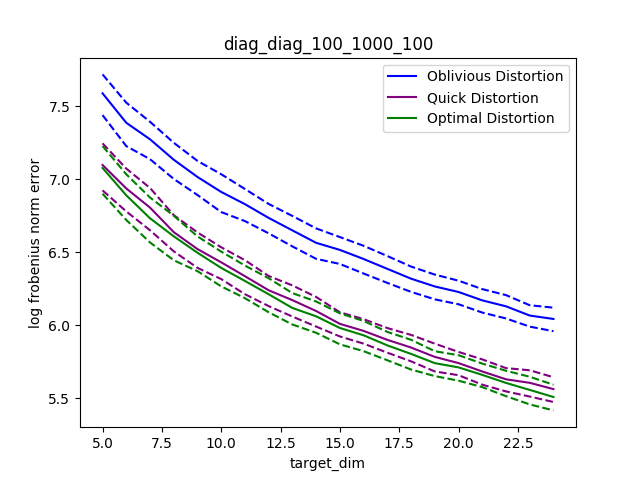} & 
  \includegraphics[width=.4\linewidth]{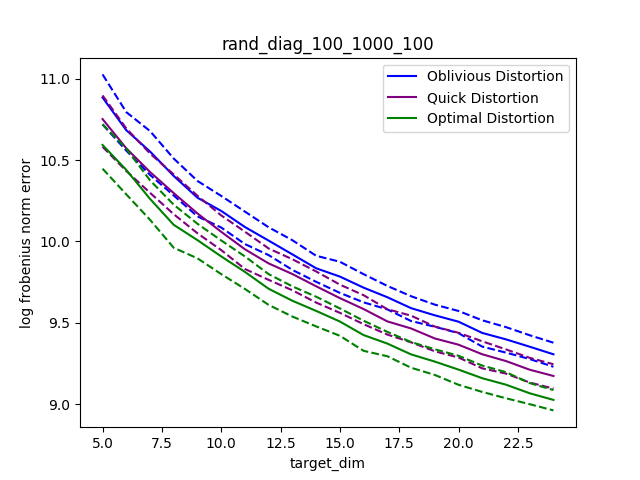} \\
  \includegraphics[width=.4\linewidth]{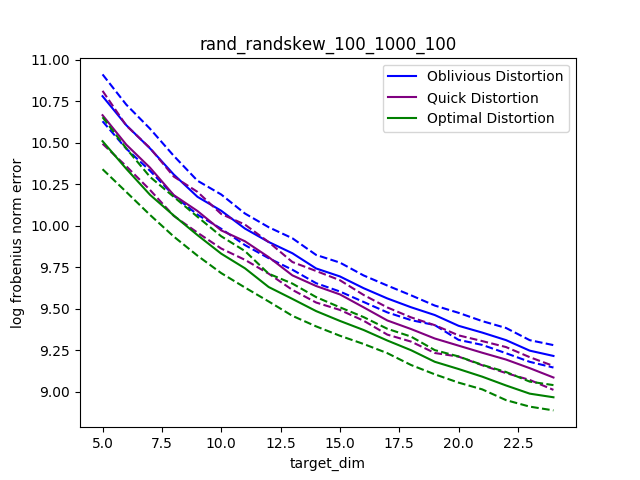}&
  \includegraphics[width=.4\linewidth]{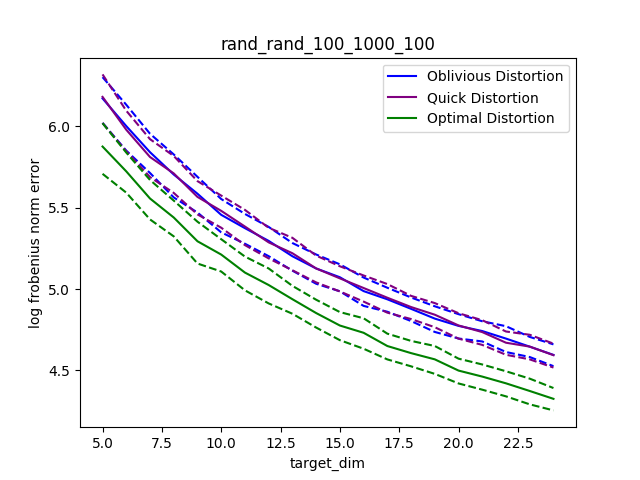}\\
  \end{tabular}  \caption{Synthetic Data FMM. Detailed in \S \ref{sec:synth}. Top left compares \textbf{diag} to \textbf{diag}. Top right compares \textbf{unif} to \textbf{diag}. Bottom left compares \textbf{unif} to \textbf{unifskew}. Bottom right compares \textbf{unif} to \textbf{unif}.}
  \label{fig:fmm_synthetic}
\end{figure}

\end{document}